\documentclass[a4paper]{article}
\newcommand{\keywords}[1]{\par\vspace{0.5em}\noindent\textbf{Keywords:} #1}
\usepackage{graphicx}
\usepackage{amsthm} 
\usepackage{mathtools}%
\usepackage{amsmath,amssymb,amsfonts}%
\DeclareMathOperator{\Tr}{Tr}%
\usepackage{enumerate}
\usepackage{crossreftools}
\usepackage{bbding}  
\usepackage{bm}
\usepackage{fancybox} 
\usepackage[hidelinks]{hyperref}%
\usepackage{cleveref}%

\usepackage{fancyhdr}

\newtheorem{assumption}{Assumption}
\newtheorem{definition}{Definition}
\newtheorem{theorem}{Theorem}
\newtheorem{lemma}{Lemma}
\newtheorem{remark}{Remark}
\newtheorem{corollary}{Corollary}
\newcommand{\orcidl}{\raisebox{-0.0ex}{\href{https://orcid.org/0000-0002-4811-6585}{\includegraphics[height=8pt]{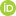}}}}
\newcommand{\orcidlo}{\raisebox{-0.0ex}{\href{https://orcid.org/0000-0003-3096-2844}{\includegraphics[height=8pt]{orcid_logo.png}}}}
\newcommand{\orcidlog}{\raisebox{-0.0ex}{\href{https://orcid.org/0000-0002-0650-3157}{\includegraphics[height=8pt]{orcid_logo.png}}}}
\newcommand{\orcidlogo}{\raisebox{-0.0ex}{\href{https://orcid.org/0000-0001-9623-8053}{\includegraphics[height=8pt]{orcid_logo.png}}}}
\newcommand{\reDEm}{\mb{E}_{\mm{Rad}^{mn},\nu_X^n}} %
\newcommand{\RE}{\mb{E}_{\mm{Rad}^{mn}}} %
\newcommand{\eR}{\widehat{\mf{R}}_n^m} 
\newcommand{\bsigma}{\bm{\sigma}} 
\newcommand{\Rad}{\mf{R}_n^m} 
\newcommand{\mbS}{\mb{S}_+^m} 
\newcommand{\mK}{\m{K}} 
\newcommand{\R}{\mb{R}} 
\renewcommand{\b}{\mathbf} 
\newcommand{\bomega}{\bm{\omega}} 
\newcommand{\kB}{k^{\scriptscriptstyle (\mm{B})}} 
\newcommand{\T}{\top} 
\newcommand{\ra}{\operatorname{ran}} 
\newcommand{\mD}{\m{D}} 
\newcommand{\mY}{\m{Y}} 
\newcommand{\mP}{\m{P}} 
\newcommand{\mU}{\m{U}} 
\newcommand{\mX}{\m{X}} 
\newcommand{\F}{\m{F}} 
\newcommand{\mH}{\m{H}} 
\newcommand{\mf}[1]{\mathfrak{#1}}

\newcommand{\mb}[1]{\mathbb{#1}}
\newcommand{\m}[1]{\mathcal{#1}}

\newcommand{\mm}[1]{\mathrm{#1}}

\newcommand{\mygraphic}[1]{\includegraphics[height=#1]{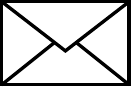}}
\newcommand{\myenv}{(\raisebox{0pt}{\mygraphic{.6em}})}

\allowdisplaybreaks

\begin{document}

\title{On the Koopman-Based Generalization Bounds for Multi-Task Deep Learning}
%
%
\author{Mahdi Mohammadigohari \inst{1}\textsuperscript{\myenv}\orcidl \and
Giuseppe Di Fatta\inst{1}\orcidlo\and
Giuseppe Nicosia\inst{3}\orcidlog\and
Panos M Pardalos\inst{4}\orcidlogo}

%
\author{
Mahdi Mohammadigohari\,\orcidl \quad
Giuseppe Di Fatta\,\orcidlo\\[2pt]
\textit{Free University of Bozen-Bolzano, Italy}\\
\href{mailto:mahdi.mohammadigohari@student.unibz.it}{mahdi.mohammadigohari@student.unibz.it} \quad
\href{mailto:giuseppe.difatta@unibz.it}{giuseppe.difatta@unibz.it}
\and
Giuseppe Nicosia\,\orcidlog\\[2pt]
\textit{University of Catania, Italy}\\
\href{mailto:giuseppe.nicosia@unict.it}{giuseppe.nicosia@unict.it}
\and
Panos M. Pardalos\,\orcidlogo\\[2pt]
\textit{University of Florida, USA}\\
\href{mailto:pardalos@ise.ufl.edu}{pardalos@ise.ufl.edu}
}

\date{}   
\maketitle              

\thispagestyle{fancy}
\fancyhf{}
\chead{%
  \footnotesize\emph{
  Accepted for LOD 2025. To appear in Lecture Notes in Computer Science (LNCS).
  }
}

\begin{abstract}
The paper establishes generalization bounds for multitask deep neural networks using operator-theoretic techniques. The authors propose a tighter bound than those derived from conventional norm-based methods by leveraging small condition numbers in the weight matrices and introducing a tailored Sobolev space as an expanded hypothesis space. This enhanced bound remains valid even in single-output settings, outperforming existing Koopman-based bounds. The resulting framework maintains key advantages such as flexibility and independence from network width, offering a more precise theoretical understanding of multitask deep learning in the context of kernel methods.

\keywords{Deep learning  \and Generalization bounds \and Multi-task learning.}
\end{abstract}
\section{Introduction}\label{sec1}
The generalization performance of deep learning has a critical impact in real-world applications; a core measure of this performance is well-defined and accurate generalization bounds. This paper aims to improve generalization bounds specifically for multi-task deep neural networks, presenting novel bounds and tools to achieve tighter estimations as opposed to existing operator and norm-based bounds. A significant theoretical approach in this area involves establishing bounds on certain complexities of these networks \cite{mohri2018foundations}. In this context, the Vapnik–Chervonenkis (VC) dimension is a key concept, serving as a measure of a model's capacity or expressive power by quantifying the number of distinct patterns that it can learn. Understanding the VC-dimension helps assess a network's potential to generalize well on new data \cite{harvey2017nearlytight}. It is commonly understood that a greater number of parameters leads to increased complexity and a higher generalization error. However, research has shown that neural networks can achieve small generalization errors even in overparameterized settings \cite{novak2018sensitivity}, where one would typically expect the error to be higher due to increased model complexity and potential overfitting. To elucidate this phenomenon, norm-based bounds have been explored \cite{Neyshabur2015,golowich2018size,bartlett2017spectrally,wei2019data,Fanghui2024}. Nevertheless, norm-based generalization bounds typically exhibit a significant dependence on the depth of the network, limiting the practical relevance of theoretical findings. Hashimoto et al. \cite{hashimoto2024koopmanbased} made progress on the generalization analysis for deep learning by introducing a novel bound based on Koopman operators. Their approach, focusing on full-rank weight matrices, demonstrated tighter bounds than existing norm-based methods, particularly when the condition numbers\footnote{The condition number of a square invertible matrix $A$ is defined as $\sigma_{\text{max}}(A) / \sigma_{\text{min}}(A)$, where $\sigma_{\text{max}}(A)$ and $\sigma_{\text{min}}(A)$ represent the largest and smallest singular values of A, respectively.} of the weight matrices are small. Furthermore, their work highlighted the benefits of orthogonal weight matrices, achieving independence from network width. However, Hashimato et al.'s framework primarily addressed single-task learning scenarios. This paper extends their operator-theoretic approach to the more complex setting of multi-task deep neural networks. Our key \textbf{contributions} are summarized as follows: 

\begin{itemize}
    \item[\textbullet] \textbf{Multi-Task Generalization Bounds.} We derive novel generalization bounds for multi-task deep neural networks, which incorporates the norm, the determinant, and the trace of an output matrix, leveraging a generic kernel framework to multi-task learning using task relations (\Cref{inj,inv}).
    
    \item[\textbullet] \textbf{Improved Single-Output Performance.} Even when restricting our framework to the single-output case, our results provide an improved generalization bound compared to existing norm-based bounds by generalizing the framework in \cite{hashimoto2024koopmanbased} to a custom space of functions (see \Cref{Rem}-(\ref{Rem_1})).
    
    \item[\textbullet] \textbf{Inherited Advantages.} We retain key advantages from the original Koopman-based bounds, including flexibility in combination with other bounds and independence from network width when using orthogonal weight matrices (see \Cref{Rem}-(\ref{Rem_2}) and (\ref{Rem_3})).
    \item[\textbullet] \textbf{Framework Extensibility.} Our analytical framework is inherently extensible. We demonstrate how it can be applied to non-injective weight matrices using graph-based and weighted Koopman approaches, and we outline a pathway for analyzing convolutional layers by recasting the convolution operation as a matrix transformation (see \Cref{Rem}-(\ref{Rem_4}) and (\ref{Rem_5})). While a detailed exploration of these extensions is beyond the scope of this paper, they highlight the versatility and adaptability of our approach.
\end{itemize}

The paper is structured as follows: \Cref{sec2} reviews related work, while \Cref{sec3} introduces key definitions, the problem setting, and notation. \Cref{sec4} presents our main theoretical results: Multi-task Rademacher complexity bounds for deep neural networks. This section generalizes the framework of Hashimoto et al. \cite{hashimoto2024koopmanbased} to a different function space, yielding tighter generalization bounds in single-output deep neural networks. We also discuss the inherited key advantages from the original Koopman-based bounds and their implications for our results. \Cref{sec5} concludes the paper.

\section{Related Works}\label{sec2}
\textbf{Norm-based generalization bounds:} Previous studies have explored generalization bounds based on the norms of weight matrices (as summarized in Table $1$ of Hashimoto et al. \cite{hashimoto2024koopmanbased}). While these bounds do not explicitly depend on layer width, they often increase as the layers become wider. Other bounds use reference matrices to explain generalization through the difference between learned weights and fixed references \cite{bartlett2017spectrally,wei2020improved,Ju2022Robust}. However, Koopman-based bound differs by focusing on the spectral properties of the weight matrices themselves, rather than comparing them to fixed references. Some bounds \cite{Li2021Orthogonal} focus only on the largest singular values of the weight matrices, while Koopman-based bound considers the entire spectrum of singular values. By extending the Koopman operator approach to multi-task learning, we can leverage task relationships to potentially achieve tighter bounds than those that treat each task independently, exposing a new aspect of multitask neural networks. \\[4pt]
\noindent\textbf{Multi-Task Learning:} The machine learning literature has extensively explored the benefits of multi-task learning \cite{argyriou2006multitask,Argyriou2008convex,di_fatta_multi_2023}. Specifically, the theoretical properties of multi-task neural networks have been a subject of recent investigation \cite{lindsey2023implicit,collins2024provable,shenouda2024variation}. Early work in multitask learning generalization includes \cite{Maurer2006Bounds}, who introduced Rademacher complexity-based bounds for linear classification. \cite{Pontil2013Excess} further analyzed Rademacher complexity to derive risk bounds in trace norm regularized models. More recently, \cite{Yousefi2018Local} established sharp risk bounds using local Rademacher complexity, examining common regularization functions. There remains a shortage of new theories for analyzing multi-task deep neural networks. To address this gap, this paper proposes a new framework to explore the generalization properties of functions learned by vector-valued (multi-output) networks, shedding new light on multi-task learning with neural networks.

\section{Preliminaries}\label{sec3}
\subsection{Notation}
For a positive integer $n$, define the set $[n]$ as $\{1, 2, \ldots, n\}$. Let $\mU$ be a topological space with a Borel sigma-field. We denote the space of probability measures on $\mU$ as $\mP(\mU)$. For a linear operator $\b W$ on a Hilbert space, its range and kernel are denoted by $\ra(\b W)$ and $\ker(\b W)$, respectively. Its operator norm is denoted by $\|\b W\|$. For a function $p \in L^{\infty}\left(\R^d\right)$, its $L^\infty$-norm is denoted by $\|p\|_\infty$. For a function $h$ on $\R^d$ and a subspace $\m{S}$ of $\R^d$, the restriction of $h$ on $\m{S}$ is denoted by $h|_\m{S}$. With $\mbS \subset \R^{m \times m}$ we denote  the set of $m \times m$ symmetric and positive semi-definite (p.s.d.) matrices.
\subsection{Vector-valued Rademacher complexity}
We consider a general multiple output regression framework. Let us briefly recall the fundamental setting of supervised learning. 
We are given a training sample $\mD_{XY,n} = \{(\b x_i, y_i)\}_{i=1}^n \sim \nu_{XY}^n$, where (i) $\nu_{XY} \in \mP(\mX \times \mY)$ denotes the joint distribution governing the relationship between the input $X$ and the output $Y$; 
(ii) $\mX\subset \R^d$ and $\mY \subset \R^m$ denote the input and output spaces, respectively; and (iii) $\nu_X$ represents the marginal distribution of $X$.\footnote{The notation $\nu_X$ will be used later when introducing Rademacher complexities.} 
The objective is to learn a function $f$ from the data $\mD_{XY,n}$ such that $f(x)$ provides an accurate prediction of the corresponding output $y$ for unseen inputs $x$. Based on the definition of Rademacher complexity, the vector-valued Rademacher complexity is defined as follows:

\begin{definition}[(empirical) vector-valued Rademacher complexity]
Let $\F$ be a class of functions $f:\mX \to \R^m$ over an input space $\mX$, and let 
\[
\bsigma_i = (\sigma_{i1},\ldots,\sigma_{im}) \sim \mathrm{Rad}^m, 
\quad i\in[n],
\]
be independent Rademacher vectors, i.e.\ $\sigma_{ij}$ are i.i.d.\ random variables uniformly distributed on $\{-1,+1\}$.  
Then, for a fixed dataset $\mD_n = \{\b x_i\}_{i=1}^n \subset \mX$, the empirical vector-valued Rademacher complexity of $\F$ is defined as
\begin{align}
\eR(\F)
\;\coloneqq\;
\RE\!\left[
\sup_{f\in\F}
\frac{1}{n}\left|\sum_{i=1}^n
\langle \bsigma_i, f(\b x_i)\rangle\right|
\right],
\label{eq:empirical-Rad-vector}
\end{align}
where $\langle\cdot,\cdot\rangle$ denotes the standard Euclidean inner product in $\R^m$.
Its expectation with respect to $(X_i)_{i=1}^n \sim \nu_X^n$ is called the (expected) vector-valued Rademacher complexity:
\begin{align}
\Rad(\F)
\coloneqq
\RE\!\left[\eR(\F)\right]
=
\reDEm\!\left[
\sup_{f\in\F}
\frac{1}{n}\left|\sum_{i=1}^n
\langle \bsigma_i, f(X_i)\rangle\right|
\right].
\label{eq:expected-Rad-vector}
\end{align}
Equivalently, one may write $\bsigma = (\bsigma_1,\ldots,\bsigma_n)\sim \mathrm{Rad}^{nm}$.
\end{definition}

\subsection{Kernel methods for vector-valued responses}
Throughout this paper, consider vector-valued Reproducing Kernel Hilbert Spaces (vvRKHSs) as function spaces. We start by recalling important notions about Matrix-Valued Kernels ($\mathrm{MVKs}$) and vvRKHS.
\begin{definition}[Matrix-valued Kernel]
Let $\mathcal{X}$ be a non-empty set. A bivariate function $K : \mathcal{X} \times \mathcal{X} \to \R^{m \times m}$, $m \in \mathbb{N}$, is called a matrix-valued kernel if $K(\b x,\b x') = K (\b x',\b x)^{\top}$ for all $(\b x,\b x') \in \mathcal{X}^2$, and for all
$n \in \mathbb{N}$ and any $(\b x_i,\b y_i)_{i=1}^n \in (\mathcal{X} \times \mathcal{Y})^n$ we have $\sum_{i,j=1}^n \b y_i^{\top}K(\b x_i,\b x_j) \b y_j \ge 0$. 
\end{definition}
\begin{theorem}[\cite{wittwar2022}, Theorem 2.2.6]
Let $K$ be a $\mathrm{MVK}$. There is a unique Hilbert space $\mH_{K} \subset \mathcal{F}(\mathcal{X} , \R^m)$,
the $\mathrm{vvRKHS}$ of $K$, such that for all $\b x \in \mathcal{X}$, $\b y \in \R^m$ and $f \in \mH_{K}$ we have $\b x' \to K(\b x,\b x')\b y \in \mH_{K}$, and
$\langle f, K(\cdot,\b x)\b y\rangle_\mH = f(\b x)^{\top}\b y$.
\end{theorem} 

\subsection{Koopman Operator}
Let $\m{H}_{K_1}$ and $\m{H}_{K_2}$ be two vector-valued RKHSs on $\R^{d_1}$ and $\R^{d_2}$, generated by the matrix-valued kernels $K_1$ and $K_2$, respectively. For a function $f : \R^{d_1} \to \R^{d_2}$, we define the Koopman operator $\mK_f$ from $\m{H}_{K_2}$ to $\m{H}_{K_1}$ by the composition as follows. Let $\m{D}_f=\{g \in \m{H}_{K_2} \,\,\,| \,\,\, g \circ \, f \in \m{H}_{K_1}\}$. The Koopman operator $\mK_f$ from $\m{H}_{K_2}$ to $\m{H}_{K_1}$ with respect to $f$ is defined as $\mK_f g = g  \circ \, f$ for $g \in \m{D}_f$.

\subsection{ Vector-Valued Sobolev Space and its Associate vvRKHS}
We set $\mathbb{N} \coloneqq \{1, 2, \ldots\}$ and $\mathbb{N}_0 \coloneqq \mathbb{N} \cup \{0\}$. Let $L^2\left(\R^d\right)$ be the space of all square-integrable functions with respect to the uniform distribution on $\R^d$, $\textit{i.e.}$ $L^2\left(\R^d\right) = \{f : \R^d \to \R, \|f\|_{L^2} < \infty\}$. 
\begin{definition}[Sobolev space]
Let $\alpha \in \mathbb{N}_0$. Then we define
\begin{align*}W^{s,2}\left(\R^d\right) \coloneqq \left\{f \in L^2\left(\R^d\right) \mid D^{\alpha}f \in L^2\left(\R^d\right) \quad \text{for all} \quad \alpha \in \mathbb{N}_0^d \quad \text{with} \quad |\alpha| \leq s\right\}\end{align*},
where $f^{\alpha}$ is the $\alpha$-th weak derivative of $f$. In this case, the Sobolev norm of $f$ is defined as
\begin{align*}
\|f\|_{W^{s,2}\left(\R^d\right)} \coloneqq \left(\sum_{0\leq |\alpha| \leq s}\|D^{\alpha}f\|^2_{L^2\left(\R^d\right)}\right)^{1/2}.    
\end{align*}
\end{definition}

\noindent The RKHS $H^s\left(\R^d\right) \coloneqq W^{s,2}\left(\R^d\right)$ will be particularly important for us. Its reproducing kernel $k_s \colon \R^d \times \R^d \to \R$ can be chosen as a translation invariant function: $k_s(\b x,\b x') = \Phi_s(\b x - \b  x')$, as long as $\Phi_s \colon \R^d \to \R$ has a Fourier transform; which decays like $(1 + \| \cdot \|_2^2)^{-s}$ (according to the Corollary $10.13$ of \cite{wendland2005scattered}).
It is also possible to choose the function $\Phi_s$ to be radial; $\textit{i.e.}$, it can be written in the form $\Phi_s = \phi_s(\| \cdot \|_2)$ with a function $\phi_s : [0, \infty ) \to \R$.

 We recall definitions and results of Sobolev space of vector-valued functions and its associated $\mathrm{vvRKHS}$. We refer the reader to \cite{wittwar2022,li2024} for more details. \\
 
\noindent
For any $f = (f_1,\ldots,f_m) \in L^2(\R^d,\R^m)$, the $L^2(\R^d,\R^m)$-norm is specified
\begin{align*}
\|f\|_{L^2(\R^d,\R^m)} = \left(\int_{\R^d} \|f(\b x)\|_2^2\mathrm{d}x \right)^{1/2}.
\end{align*}
We have the following characterization for Sobolev space.
\begin{definition}[ Vector-Valued Sobolev Space]
\label{vvsp}    
The Sobolev space of order $s \ge 0$ over $\R^d$ is defined by
\begin{align}\label{VVSS}
 H^s(\R^d,\R^m) \coloneqq \left\{f \in L^2(\R^d,\R^m) \mid (1 + \|\bomega\|_2^2)^s \hat{f} \in L^2(\R^d,\R^m)\right\},   
\end{align}
where the Fourier transform of a function here is defined componentwise$,\textit{i.e.}$, integral operates on each individual component.
This is again a Hilbert space equipped with the inner product 
\begin{align*}\langle f,g \rangle_{H^s(\R^d,\R^m)} = \int_{\R^d} (1 + \|\bomega\|_2^2)^s \hat{f}(\bomega)^*\hat{g}(\bomega)\mathrm{d}\bomega.
\end{align*}
\end{definition}
Here, $\hat{\cdot}$ represents the Fourier transform, and $*$ denotes the complex conjugate (or Hermitian transpose). The following lemma (\cite{wittwar2022}, Corollary $2.3.17$) gives that the native space $\m{H}_{K_s}$ for $K_s(\b x,\b x') = \boldsymbol{\Phi}_s(\b x - \b  x')$  coincides with $H^s(\R^d,\R^m)$ if the Fourier transform of
$\boldsymbol{\Phi}_s$ meets certain conditions. 

\subsection{Integral Space of RKHSs}
This subsection revisits the idea of the integral space of Reproducing Kernel Hilbert Spaces (RKHSs) as a generalized hypothesis space \cite{Hotz2012RepresentationBI,he2024learninganalysiskernelridgeless}. We consider the Brownian kernel $\kB(x,x') = \frac{|x| + |x'| - |x - x'|}{2}$, defined for $x, x' \in \R$.
Let $\mX$ be a domain and $\Omega$ be an index set with a topological structure equipped with a Borel probability measure $\mu \in \mP(\Omega)$. We introduce the following kernel:
\begin{align*}
k^{(\bomega)}(\b x,\b x') \coloneqq k^{\scriptscriptstyle(\mathrm{B})}(\bomega \b x,\bomega \b x') = \frac{|\bomega x| + |\bomega x'| - |\bomega x - \bomega x'|}{2},
\end{align*}
for each $\omega \in \Omega$ and $x,x'\in \mX$. The Sobolev RKHS associated with $k^{(\bomega)}$, denoted as $\m{H}_{k^{(\bomega)}}$ (also referred to as $\m{H}(\Omega)$), is
\begin{align*}
    \m{H}_{k^{(\bomega)}} = \left\{g : \R \to \R \middle|\, g(0) = 0, \, g'(x) = \mathrm{d}g/\mathrm{d}x,\, \int_{\R} g'(x)^2\mathrm{d}x < \infty\right\}.
\end{align*}
For any $\b x \in \mX$, $k^{(\bomega)}_x$ is an element of $\m{H}(\Omega)$, and we denote functions in $\m{H}(\Omega)$ by $f_{\bomega}$. The integral space of RKHSs defined over $\Omega$ can be described as:
\begin{align}\label{IRKHS}
\m{H}(\Omega) = \left\{ f = (f_{\bomega})_{\bomega \in \Omega} \middle|\,\int_{\Omega}\|f_{\bomega}\|_{\m{H}(\Omega)}^2\mu(\mathrm{d}\bomega) < \infty \right\},
\end{align}
where $(f_{\bomega})_{\bomega \in \Omega}$ represents a measurable cross-section. It has been shown that $\m{H}(\Omega)$ is a Hilbert space equipped with the inner product:
\begin{align*}
\langle f,g \rangle_{\m{H}(\Omega)} \coloneqq \int_{\bomega \in \Omega} \langle f_{\bomega}, g_{\bomega} \rangle_{\m{H}(\Omega)} \mu(\mathrm{d}\bomega), \quad f = (f_{\bomega})_{\bomega \in \Omega}, g = (g_{\bomega})_{\bomega \in \Omega} \in \m{H}(\Omega).
\end{align*}
Thus, it follows that: $f(\b x) = \int_{\bomega \in \Omega} f_{\bomega}(\b x) \mu(\mathrm{d}\bomega), \quad \text{for all } \b x \in \mX$. The associated norm is defined as:
\begin{align*}
   \|f\|_{\m{H}(\Omega)}^2 \coloneqq \inf \left\{ \int_{\bomega \in \Omega} \|f_{\bomega}\|_{\m{H}(\Omega)}^2 \mu(\mathrm{d}\bomega) \mid f = (f_{\bomega})_{\bomega \in \Omega} \right\},
\end{align*}
where $\|f_{\bomega}\|^2_{\m{H}(\Omega)} = \langle f_{\bomega}, f_{\bomega} \rangle_{\m{H}(\Omega)}$.
\begin{figure}[htbp]
  \centering
  \includegraphics[width=0.8\linewidth]{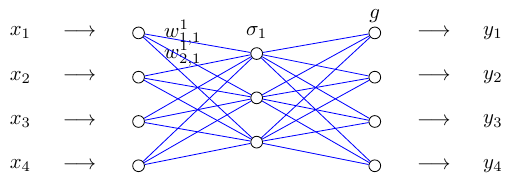}
  \caption{Illustration of the proposed network architecture. The network consists of an input layer, one hidden layer, activation function $\sigma_1$, final nonlinear transformation $g$,  and an output layer.}
  \label{fig:myfig}
\end{figure}

\subsection{Problem Setup}\label{setting}
This paper analyzes the generalization properties of $L$-layer deep neural networks. We consider networks that map from an input space $\R^{d_0}$ to an output space $\R^m$. The network architecture is defined as:
\begin{equation} \label{eq:network_architecture}
f = g \circ \b b_L \circ \b W_L \circ \sigma_{L-1} \circ \b b_{L-1} \circ \b W_{L-1} \circ \ldots \circ \sigma_1 \circ \b b_1 \circ \b W_1,
\end{equation} 
where:
\begin{itemize}
    \item $\b W_l \in \R^{d_l \times d_{l-1}}$ ($l=1,\ldots,L$): Represents the weight matrix of the $l$-th layer, mapping from $\R^{d_{l-1}}$ to $\R^{d_l}$. We assume each $\b W_l$ is injective (or invertible, depending on the specific analysis).
    \item $\b b_l$ ($l=1,\ldots,L$): A shift operator applied to the $l$-th layer, defined as $\b b_l(\b x) = x + a_l$ where $\b a_l \in \R^{d_l}$ is a bias vector.
    \item $\sigma_l$ ($l=1,\ldots,L-1$): The activation function of the $l$-th layer, a nonlinear mapping $\sigma_l: \R^{d_l} \to \R^{d_l}$.
    \item $g: \R^{d_L} \to \R^m$: The final nonlinear transformation mapping the output of the $L$-th layer to the $m$-dimensional output space.
\end{itemize}

In this work, we extend the operator-theoretic framework established in \cite{hashimoto2024koopmanbased}, thereby enabling a comprehensive analysis of vector-valued neural networks through the lens of vector-valued reproducing kernel Hilbert spaces (RKHSs). With $d_0$ as the input space dimension and $d_l$ as the width for $l = 1,\ldots, L$, the $L$-layer deep neural network $f$ is defined as 
\begin{align}\label{Net:1}
  f =  g  \circ \, \b b_L  \circ \, \b W_L  \circ \, \sigma_{L-1} \circ \, \b b_{L-1}  \circ \, \b W_{L-1}  \circ \, \ldots  \circ \, \sigma_1  \circ \, \b b_1  \circ \, \b W_1,  
\end{align} 
We consider vvRKHSs $H^{s_l}(\R^{d_l}, \R^m)$, generated by a kernel $K_{s_l} = k_{s_l} M$, where $k_{s_l}$ is a radial basis scalar kernel, $s_l \ge 0$, $s_l > d_l/2$, and $\b M \in \mathbb{S}_+^m$. We can represent the network using the product of Koopman operators.
\begin{center}  
$g(\b x)
\xrightarrow{\mK_{\b b_L}} 
g(\b x + \b b_L) 
\xrightarrow{\mK_{\b W_L}} 
g(\b W_L \b x + \b b_L)$

\vspace{0.5em}

$\phantom{g(\b x)}
\quad 
\quad\quad\quad
\xrightarrow{\mK_{\b W_{L-1}} \mK_{\b b_{L-1}} \mK_{\sigma_{L-1}}}
 g(\b W_{L} \sigma_{L-1}(\b W_{L-1} \b x + \b b_{L-1}) + \b b_L)$

\vspace{0.5em}
\makebox[8cm][l]{$\quad \quad \quad \cdots \longrightarrow$}
\end{center}
Then, the network $f$ can be described by
\begin{align}\label{koopnet}
f = \mK_{\b W_1} \mK_{\b b_1} \mK_{\sigma_1} \ldots \mK_{\b W_{L-1}} \mK_{\b b_{L-1}}
\mK_{\sigma_{L-1}} \mK_{\b W_L} \mK_{\b b_L} g,    
\end{align}
where 
\begin{align*}\mK_{\b W_l} \colon H^{s_l}\left(\R^{d_l},\R^m\right) \to H^{s_{l-1}}\left(\R^{d_{l-1}},\R^m\right),
\end{align*}
and 
\begin{align*}\mK_{\b b_l}, \mK_{\sigma_l} \colon H^{s_l}\left(\R^{d_l},\R^m\right) \to H^{s_l}\left(\R^{d_l},\R^m\right).
\end{align*}
\noindent The architecture of the shallow network is illustrated in \Cref{fig:myfig}.
\section{Koopman-based Bound of Multi-task Rademacher Complexity}
\label{sec4}
This section focuses on multi-task learning, specifically addressing the multiple-label learning scenario. We derive Koopman-based generalization bounds for multi-task deep neural networks.  Vector-valued regression (VVR), where all outputs are available during training for each input, serves as a specific instance of multi-task learning (MTL) \cite{ElAhmad2023}.  We consider $T$ classification tasks, each with data $(\b x_t, \b y_t)$ drawn from a distribution $P_t$ over $\mX \times \{1, 2, \ldots, m\}$, where $\mX \subset \R^d$ is a compact input space and ${1, 2, \ldots, m}$ represents the $m$ possible classes. Following \cite{Evgeniou2006}, we assume that the input data $x_{ti}$ is independent of the task index $t$, allowing us to treat the input space as a Cartesian product $\mX = \mX_1 \times \mX_2 \times \ldots \times \mX_T$ for all $T$ tasks. We leverage task relationships to construct a general kernel approach for multi-task learning. For each task $t$, we consider a vvRKHS, denoted as $H_t^{s_l}(\R^{d_l}, \R^m)$, generated by a kernel $K_{ts_l} = k_{ts_l} \b M_t$, where $k_{ts_l}$ is a radial basis scalar kernel, $s_l > d_l/2$, and $\b M_t \in \mathbb{S}_+^m$.  To ensure task independence, we assume that the output matrices $\b M_t$ are diagonal. The overall function space for the network is then the direct sum of these individual task spaces
\begin{align*}
H^{\oplus s_l}\left(\R^{d_l},\R^m\right) = \bigoplus\limits_{t = 1}^{T} \m{H}_{K_{ts_l}} = \bigoplus\limits_{t = 1}^{T} H^{s_l}_t\left(\R^{d_l},\R^m\right),\quad \text{for}\,\,l=1,\ldots,L    
\end{align*}
with reproducing kernel $K_{s_l} = \sum_{t=1}^T K_{ts_l}$. Each $f \in H^{\oplus s_l}\left(\R^{d_l},\R^m\right)$ admits a unique orthogonal decomposition $f = \sum_{i=1}^T f_t$, $f_t \in H^{s_l}_t\left(\R^{d_l},\R^m\right)$, 
and the corresponding norm is defined as
\begin{align*}
\|f\|_{H^{\oplus s_l}\left(\R^{d_l},\R^m\right)} = \left(\sum_{t = 1}^T\|f_t\|^2_{H^{s_l}_t\left(\R^{d_l},\R^m\right)}\right)^{1/2}.    
\end{align*}
The construction of our network is now achieved as follows. The network $f$ can be described by the componentwise Koopman operators given as \eqref{koopnet},  where \begin{align*}\mK_{\b W_l} : H^{\oplus s_l}\left(\R^{d_l},\R^m\right) \to H^{\oplus s_{l-1}}\left(\R^{d_{l-1}},\R^m\right),
\end{align*}
and 
\begin{align*}\mK_{\b b_l}, \mK_{\sigma_l} : H^{\oplus s_l}\left(\R^{d_l},\R^m\right) \to H^{\oplus s_l}\left(\R^{d_l},\R^m\right).
\end{align*}
 We make the following assumption we will use throughout this section.
\begin{assumption}\label{assum:4}
The final nonlinear transformation $g = (g_1,\ldots,g_T)$ is contained in 
$H^{\oplus s_L}\left(\R^{d_l},\R^m\right)$,
and $\mK_{\sigma_l}$ are bounded for $l = 1,\ldots, L-1$. 
\end{assumption}
To ensure that the final nonlinear transformation $g$ meets \Cref{assum:4}, we define $g(\b x) = \sum_{t=1}^T e^{-r_t \|\b x\|^2} \b M_t \b c_t^\top$ for $\b x \in \R^{d_L}$, with coefficients $\b c_1, \dots, \b c_T \in \R^m$, $r_1,\ldots,r_T \in \mathbb{N}$, $M_1,\ldots,\b M_t \in \mathbb{S}_+^m$, and $s_L > d/2$. With this definition, $g$ is an element of $H^{\oplus s_L}(\R^{d_L}, \R^m)$. Let $\F$ be the set of all functions in the form \eqref{Net:1} with \Cref{assum:4}. 
\begin{remark}
Let $g$ be a smooth function which does not decay at infinity, (e.g., sigmoid). Although $\mH^{\oplus s_L}(\R^{d_L}, \R^m)$ does not contain $g$, we can construct a function $\tilde{g} \in \mH^{\oplus s_L}(\R^{d_L}, \R^m)$ such that $\tilde{g}(\b x) = g(\b x)$ for $x$ in a sufficiently large compact region and replace $g$ by $\tilde{g}$ in practical cases.
\end{remark}
\begin{assumption}\label{assum:2}
There exists $k > 0$, such that for all $\b x \in \mX$, we have $k_{s_0}(\b x,\b x) \leq k$.    
\end{assumption}
Regarding the activation function, the proof of the following proposition follows a direct modification of the proof of Proposition 1 in \cite{hashimoto2024koopmanbased}.
\begin{lemma}
For $l = 1,\ldots,L-1$, let $\sigma_l$ be bi-Lipschitz, $\textit{i.e.}$, $\sigma_l$ is bijective and both $\sigma_l$ and $\sigma_l^{-1}$ are Lipschitz continuous, $s$-times differentiable and its derivative $\partial^{\alpha}\sigma_l$ be bounded for any multi-index \begin{align*}
\alpha \in \left\{(\alpha_1, \ldots , \alpha_d) \mid \alpha_1 + \ldots + \alpha_d \leq s \right\}.
\end{align*}
Also, for $l = 0,\ldots,L$, let $s_l \in \mathbb{N}$, and $s_l > d_l/2$ be the smoothness of the Sobolev spaces $H^{\oplus s_l}\left(\R^{d_l},\R^m\right)$. Then $\mK_{\sigma_l} \colon H^{\oplus s_l}\left(\R^{d_l},\R^m\right) \to H^{\oplus s_l}\left(\R^{d_l},\R^m\right)$ are bounded for $l = 1,\ldots,L-1$. 
\end{lemma} 

\noindent For the activations $\sigma_l$-s, we can opt to use a smoothed variant of the Leaky ReLU activation function \cite{biswas2022smooth}.

Let $\F$ denote the set of all functions adhering to the architecture specified in \eqref{eq:network_architecture} and satisfying Assumption \ref{assum:4}. We proceed to derive Rademacher complexity bounds for this function class.

\subsection{Bound for invertible weight matrices}\label{sub:1}
In this subsection, for some $d \in \mathbb{N}$, we assume $d_l = d$ for $l = 0,\ldots,L$. For $C,D > 0$, consider a class of weight matrices $\m{W}^{\, \scriptscriptstyle C,D} = \{\b W\in \R^{d\times d} \mid \|\b W\| \leq C,\, |\det(\b W)| \ge D\}$, and $\F_{\scriptscriptstyle \mathrm{inv}}= \{f \in \F \mid \b W_l\in \m{W}^{\,\scriptscriptstyle C,D}\}$. Then, the Rademacher complexity with respect to the function class  $\F_{\scriptscriptstyle\mathrm{inv}}$ can be bound as follows.
\begin{theorem}
\label{inv}
The Rademacher complexity $\Rad(\F_{\scriptscriptstyle \mathrm{inv}})$ is bounded as
 \begin{align*}
&\eR(\F_{\scriptscriptstyle \mathrm{inv}})\leq \left(T\sqrt{\frac{\kappa U_0}{n}}\right)\|g\|_{H^{\oplus s_L}(\R^d,\R^m)}\\
&\cdot \sup_{\b W_l \in \m{W}^{\,\scriptscriptstyle C,D}} \prod_{l = 1}^L \sup_{\bomega \in \R^d} \left|\frac{(1 + \|\b W_l^{\top} \bomega\|^2_2)^{s_{l-1}}}{(1 + \|\bomega\|^2_2)^{s_l}}\right|^{1/2} \frac{1}{|\det(\b W_l)|^{1/2}}\prod_{l = 1}^{L-1}\|\mK_{\sigma_l}\|,
\end{align*}
where $\b k_{t0} \in \R^{n \times n}$ is the scalar Gram matrix of $k_{ts_0}$, $\b K_{t0} = \b k_{t0} \otimes \b M_t$ is the Gram matrix of $K_{ts_0}$, and $U_0 \coloneqq \sum_{t=1}^T \sqrt{\Tr(\b M_t)}$.
\end{theorem}
\begin{proof}
Using the shift invariant property of the Fourier transform, we have $\|\mK_{\b b_l}\| = 1$. Indeed, for $h = (h_1,\ldots,h_T) \in H^{\oplus s_l}\left(\R^{d_l},\R^m\right)$, we have $\left(\widehat{h \circ \, \b b_l}\right)(\bomega) = e^{-i\b a_l^{\top}\cdot\bomega}\hat{h}(\bomega)$. Thus, we obtain
$\|\mK_{b}h\|^2_{H^{\oplus s_l}\left(\R^{d_l},\R^m\right)} = \|h\|^2_{H^{\oplus s_l}\left(\R^{d_l},\R^m\right)} $.
In addition, for $h \in H^{\oplus s_l}\left(\R^{d_l},\R^m\right)$, by the scale property of Fourier transform, we have 
$\left(\widehat{h \circ \, \b W_l}\right)(\bomega) = (1/|\det(\b W_l)|)\hat{h}\left(\b W_l^{-\top}\bomega\right)$. 
Therefore, we deduce
\begin{align}
\label{kbound1}
&\|\mK_{\b W_l}h\|^2_{H^{\oplus s_{l-1}}\left(\R^{d_l},\R^m\right)} = \sum_{t=1}^T \|h_t \circ\, \b W_l\|_{H^{s_l}_t\left(\R^{d_l},\R^m\right)} \nonumber\\
&= \sum_{t = 1}^T \int_{\R^d} (1 + \|\bomega\|^2_2)^{s_{l-1}}(\widehat{h_t \circ \, \b W_l})(\bomega)^* (\widehat{h_t \circ \, \b W_l})(\bomega) \mathrm{d}\bomega \nonumber\\
&\leq \left(\sum_{t=1}^T \|h_t\|^2_{H^{s_l}_t\left(\R^{d_l},\R^m\right)} \sup_{\bomega \in \R^d}\left|\frac{(1 + \|\b W_l^{\top} \bomega\|^2_2)^{s_{l-1}}}{(1 + \|\bomega\|^2_2)^{s_l}}\right|\right)\frac{1}{|\det(\b W_l)|}\nonumber\\
&= \left(\|h\|^2_{H^{\oplus s_l}\left(\R^{d_l},\R^m\right)} \sup_{\bomega \in \R^d}\left|\frac{(1 + \|\b W_l^{\top} \bomega\|^2_2)^{s_{l-1}}}{(1 + \|\bomega\|^2_2)^{s_l}}\right|\right)\frac{1}{|\det(\b W_l)|},
\end{align}
Let $x_1,\ldots,x_n \in \R^{d_0}$. Since $f = (f_1,\ldots,f_T) \in H^{\oplus s_0}\left(\R^{d_0},\R^m\right)$, by applying the reproducing property of $H^{\oplus s_0}\left(\R^{d_0},\R^m\right)$, we have
\begin{align}
&\frac{1}{n}\RE \left[\sup_{\substack{ f \in \F}}  \left|\sum_{i=1}^n \langle\bsigma_{i}, f(\b x_i)\rangle_{\R^m}\right|\right] = \frac{1}{n}\RE \left[\sup_{\substack{ f \in \F}}  \left|\sum_{i=1}^n \langle\bsigma_{i}, \sum_{t=1}^T f_t(\b x_i)\rangle_{\R^m}\right|\right] \nonumber\\
&= \frac{1}{n}\RE \left[\sup_{\substack{ f \in \F}} \sum_{t=1}^T  \langle \sum_{i=1}^n \bsigma_i K_{ts_0}(\cdot,\b x_i),f_t\rangle_{H^{s_0}_t\left(\R^{d_0},\R^m\right)} \right]\nonumber\\
&\leq \frac{1}{n} \RE \left[\sup_{\substack{ f \in \F}}\left(\sum_{t=1}^T \left\|\sum_{i=1}^n \bsigma_i  K_{ts_0}(\cdot,\b x_i)\right\|_{H^{s_0}_t\left(\R^{d_0},\R^m\right)}\right) \left(\sum_{t=1}^T \|f_t\|_{H^{s_0}_t\left(\R^{d_0},\R^m\right)}\right)\right]\label{th6proofeq1}\\
&\leq \frac{1}{n} \RE \left[\sup_{\substack{ f \in \F}}\left(\sum_{t=1}^T \left\|\sum_{i=1}^n \bsigma_i  K_{ts_0}(\cdot,\b x_i)\right\|^2_{H^{s_0}_t\left(\R^{d_0},\R^m\right)}\right)^{1/2} \vphantom{\cdot \left(\sum_{t=1}^T \|f_t\|^2_{H^{s_0}_t\left(\R^{d_0},\R^m\right)}\right)^{1/2}}\right.\nonumber\\
&\qquad \qquad \qquad \qquad \qquad \qquad \qquad \qquad \left. \cdot \left(\sum_{t=1}^T \|f_t\|^2_{H^{s_0}_t\left(\R^{d_0},\R^m\right)}\right)^{1/2}\right]\label{th6proofeq2}\\
&\leq \frac{1}{n} \RE \left[\sup_{\substack{ f \in \F}}\|f\|_{H^{\oplus s_0}\left(\R^{d_0},\R^m\right)}\left(\sum_{t=1}^T \left\|\sum_{i=1}^n \bsigma_i  K_{ts_0}(\cdot,\b x_i)\right\|_{H^{s_0}_t\left(\R^{d_0},\R^m\right)}\right) \right]\label{th6proofeq3}\\
&= \frac{1}{n} \sup_{\substack{ f \in \F}}\|f\|_{H^{\oplus s_0}\left(\R^{d_0},\R^m\right)}\left(\sum_{t=1}^T \RE\left\|\sum_{i=1}^n \bsigma_i  K_{ts_0}(\cdot,\b x_i)\right\|_{H^{s_0}_t\left(\R^{d_0},\R^m\right)}\right) \nonumber\\
&\leq \frac{1}{n}\sup_{\substack{ f \in \F}}\|f\|_{H^{\oplus s_0}\left(\R^{d_0},\R^m\right)} \sum_{t=1}^T \left(\Tr(\b K_{t0})\right)^{1/2}\nonumber\\
&= \frac{1}{n} \sum_{t=1}^T \sqrt{\Tr(\b K_{t0})}\sup_{\substack{ \b W_l}} \|\mK_{\b W_1} \mK_{\b b_1} \mK_{\sigma_1} \ldots \mK_{\b W_L} \mK_{\b b_L} g\|_{H^{\oplus s_0}\left(\R^{d_0},\R^m\right)}\nonumber\\
&\leq\frac{1}{n} \sum_{t=1}^T \sqrt{\Tr(\b K_{t0})}\sup_{\b W_l} \left(\prod_{l=1}^L\|\mK_{\b W_l}\|\|\mK_{\b b_l}\| \prod_{l=1}^{L-1} \|\mK_{\sigma_l} \|\right) \|g\|_{H^{\oplus s_L}\left(\R^{d_l},\R^m\right)} \nonumber\\
&\leq \frac{1}{n} \sum_{t=1}^T \sqrt{\Tr(\b K_{t0})}\sum_{t=1}^T \sqrt{\Tr(\b M_t)}\sup_{\b W_l} \left(\prod_{l=1}^L\|\mK_{\b W_l}\| \prod_{l=1}^{L-1}\|\mK_{\sigma_l} \|\right) \|g\|_{H^{\oplus s_L}\left(\R^{d_l},\R^m\right)}
\nonumber\\
&\leq T\sqrt{\frac{\kappa}{n}} \sum_{t=1}^T \sqrt{\Tr(\b M_t)}\sup_{\b W_l} \left(\prod_{l=1}^L\|\mK_{\b W_l}\| \prod_{l=1}^{L-1}\|\mK_{\sigma_l} \|\right) \|g\|_{H^{\oplus s_L}\left(\R^{d_l},\R^m\right)}\label{ab} 
\end{align}
where \eqref{th6proofeq1} results from the Cauchy–Schwarz inequality, \eqref{th6proofeq2} follows by the H\"{o}lder inequality and \eqref{th6proofeq3} is derived by using the sequence spaces property of $\ell^p$ (\textit{i.e.,} $\|\b x\|_q \leq \|\b x\|_p$ for $1 \leq p \leq q \leq \infty$). By \eqref{kbound1}, we obtain the final result. \hfill$\square$    
\end{proof}

The bound in \cref{inv} is derived using a Koopman operator representation of the network.  This allows us to obtain the overall bound by individually bounding the norms of each Koopman operator. A key distinction of our bound lies in the presence of determinant factors in the denominator.
As a direct application of the Lemma 5 of \cite{hashimoto2024koopmanbased} to \Cref{inv}, we have the following result.
\begin{corollary}
\label{cor1}
Let $H^{\oplus s_l}\left(\R^d, \R^m\right) =H^{\oplus s}\left(\R^d, \R^m\right)$ for $l=0,\ldots,L$ and $s > d/2$. Then, we have 
 \begin{align*}
\eR(\F_{\scriptscriptstyle \mathrm{inv}}) \leq \max(1,C^s)\left(T\sqrt{\frac{\kappa U_0}{nD}}\right) \|g\|_{H^{\oplus s}(\R^d,\R^m)}\prod_{l = 1}^{L-1} \|\mK_{\sigma_l}\|. 
 \end{align*}     
\end{corollary}
\subsection{Bound for injective weight matrices}\label{sub:2}
We extend the result of \Cref{inv} to the more general case of injective weight matrices. Let $C,D > 0$, set the weight matrices class as 
\begin{align*}\m{W}_l^{\,\scriptscriptstyle C,D} = \left\{\b W\in \R^{d_{l-1}\times d_l} \mid d_l\ge d_{l-1} \,, \|\b W\| \leq C,\, \det(\b W^{\top}\b W)^{1/2} \ge D\right\},
\end{align*}
and $\F_{\scriptscriptstyle \mathrm{inj}}= \left\{f \in \F \mid \b W_l\in \m{W}_l^{\,\scriptscriptstyle C,D}\right\}$. Let $f_l = g  \circ \, \b b_L  \circ \, \b W_L  \circ \, \sigma_{L-1} \circ \, \b b_{L-1}  \circ \, \b W_{L-1}  \circ \, \ldots  \circ \, \sigma_l  \circ \, \b b_l$ and $\mathrm{G}_l =\frac{\|f_l|_{\ra(\b W_l)}\|_{H^{\oplus s_{l-1}}(\ra(\b W_l), \R^m)}}{\|f_l\|_{H^{\oplus s_l}\left(\R^{d_l},\R^m\right)}}$. We have the following theorem for a bound of Rademacher complexity with respect to  $\F_{\scriptscriptstyle \mathrm{inj}}$.
The multi-task Rademacher complexity of $\F_{\scriptscriptstyle \mathrm{inj}}$ can be computed by similar approach and means.
\begin{theorem}\label{inj}
The Rademacher complexity $\eR(\F_{\scriptscriptstyle \mathrm{inj}})$ is bounded as
 \begin{align*}
 &\eR(\F_{\scriptscriptstyle \mathrm{inj}}) \leq \left(T\sqrt{\frac{\kappa U_0}{n}}\right)\|g\|_{H^{\oplus s_L}\left(\R^{d_l},\R^m\right)} \\
 &\cdot\sup_{\b W_l \in \m{W}_l^{\,\scriptscriptstyle C,D}} \prod_{l = 1}^L \mathrm{G}_l \sup_{\bomega \in \ra(\b W_l)} \left|\frac{1 + \|\b W_l^{\top} \bomega\|^2_2}{1 + \|\bomega\|^2_2}\right|^{s_{l-1}/2} \frac{1}{\det(\b W_l^{\top}\b W_l)^{1/4}}\prod_{l = 1}^{L-1}\|\mK_{\sigma_l}\|.
 \end{align*}    
\end{theorem} 
\Cref{inj} extends the generalization bound to injective weight matrices, addressing the scenario where the weight matrices are not square. While the determinant term here is no longer directly applicable, we can use $\det(\b W_l^{\top} \b W_l)^{1/2}$ as a suitable replacement.
\begin{remark}\label{Rem}
  \leavevmode\begin{enumerate}[(i)]
 \item For simplicity, we consider single output neural networks and RKHSs associated with one-dimensional Brownian kernel as function spaces. Now, considering the function space defined by $ \eqref{IRKHS} $ over the unit sphere $ \m{S}^{d-1} $ for some norm $ \|\cdot\| $ in $ \R^d $ (typically either the $ \ell_2 $ or $ \ell_1 $ norm), we use a specialized norm given by:
\begin{align*}
\Theta(f) \coloneqq \inf \left\{ \int_{\m{S}^{d-1}} \|f_{\bomega}\|_{\m{H}(\m{S}^{d-1})} \mu(\mathrm{d}\bomega) \mid f = (f_{\bomega})_{\bomega \in \m{S}^{d-1}} \right\},
\end{align*}
with further details and properties available in \cite{follain2024enhanced}. We establish the following inequality
\begin{align*}
\|\mK_\sigma\| \leq \| \det(J_{\sigma^{-1}}) \|_\infty \max \{\|\partial_1 \sigma\|_\infty, \ldots, \|\partial_d \sigma\|_\infty \},
\end{align*}
when $ s = 1 $ and $ \sigma $ is applied elementwise, where $ J_{\sigma^{-1}} $ represents the Jacobian of $ \sigma^{-1} $. Consequently, the empirical Rademacher complexity $ \eR(\F_{\scriptscriptstyle \mathrm{inv}}) $ is bounded as:
\begin{align}
\m{O}\left(
\sup_{\b W_l \in \m{W}^{\,\scriptscriptstyle C,D}} \prod_{l = 1}^L\|\b W_l\| \prod_{l = 1}^{L-1}\|\det(J_{\sigma_l^{-1}}) \|_\infty \max \{\|\partial_1 \sigma_l\|_\infty, \ldots, \|\partial_d \sigma_l\|_\infty \}\right).
\end{align}
It is important to note that our bound is tighter than the following alternative bound:
\begin{align*}
&\m{O}\left(
\sup_{\b W_l \in \m{W}^{\,\scriptscriptstyle C,D}} \prod_{l = 1}^L\frac{\max(1, \|\b W_l\|^2)^{1/2}}{\det(\b W_l^\T \b W_l)^{1/4}}\vphantom{\sup_{\b W_l \in \m{W}^{\,\scriptscriptstyle C,D}} \prod_{l = 1}^L\frac{\max(1, \|\b W_l\|^2)^{1/2}}{\det(\b W_l)^{1/2}}} \right. \\
& \qquad \qquad \qquad \left.\cdot \prod_{l = 1}^{L-1}\|\det(J_{\sigma_l^{-1}}) \|_\infty \max \{1,\|\partial_1 \sigma_l\|_\infty, \ldots, \|\partial_d \sigma_l\|_\infty \}\right),
\end{align*}
as derived in \cite{hashimoto2024koopmanbased}. They showed that their bound is tight when the condition number of the weight matrices is small, hence, our bound provides a further refinement and improvement in this regime. The direct nature of the operator norm allows for a more accurate depiction of the matrix's effect compared to a bound influenced significantly by determinants. \label{Rem_1}
\item Note that the factors $\|\b W_l\|$ in our bound can be bounded by $1$ for an orthogonal weight matrix $\b W_l$. In this case, our bound does not depend on the width of the layer. For more details, we refer the readers to Remark $6$ of \cite{hashimoto2024koopmanbased}.\label{Rem_2}
\item In \cite{hashimoto2024koopmanbased}, the authors demonstrated that the Koopman-based bound is sufficiently flexible to be combined with another bound. Furthermore, the conditions under which the Koopman-based bound is tight differ from those for existing bounds, such as the one established by \cite{golowich2018size}. By combining these bounds, we can achieve a tighter overall bound.\label{Rem_3}
\item Upper bounds of the multi-task Rademacher complexity can be obtained even for non-injective weight matrices using the same graph-based and weighted Koopman approaches (see, \cite{hashimoto2024koopmanbased}, Section 4.3). To avoid prolonging the analysis, we do not include them.\label{Rem_4}
\item Our framework extends to convolutional layers by regarding the convolution as the action of a matrix. We construct a tensor $\b W_{i,j,k,l} = f_{k-i,l-j}$ from filter $F = [f_{i,j}]$ for convolution $\sum_{i=1}^n\sum_{j=1}^mf_{k-i,l-j}x_{i,j}$; out-of-bounds indices are zeroed. This tensor is then reorganized as a matrix, allowing analysis via our existing theory. While pooling layers are not currently addressed, analyzing them represents a key direction for our future research.\label{Rem_5}
\end{enumerate}
\end{remark}

\section{Conclusion}\label{sec5}
This paper introduces an operator-theoretic approach to analyze the generalization theory of multi-task deep neural networks. We derive novel uniform complexity bounds using Koopman operators, encompassing weight matrix norms and determinants, and the trace of an output matrix. Our bounds rely on assumptions regarding the final nonlinear transformation, Koopman operator boundedness, and kernel function boundedness. In this work, we considered the Koopman operator for each transformation. However, combining several layers that have similar roles together and considering the Koopman operator for the combined transformation may alleviate the dependency on the width of the intermediate layers. Thus, we may have room for deriving an even tighter bound by combining several layers together. As of now, our analysis is independent of the existing analysis of learning dynamics with Koopman operators. Combining our results with the analysis of learning dynamics is an interesting direction for future work. To solidify the impact of this work, a crucial next step is the comprehensive empirical validation of the derived bounds and proposed models across diverse datasets and applications, thereby bridging the gap between theoretical insights and tangible improvements in multi-task deep learning.

\bibliographystyle{splncs04}
\bibliography{sn-bibliography}

\end{document}